\def\year{2017}\relax
\documentclass[letterpaper]{article}

\usepackage{aaai17}
\usepackage[utf8]{inputenc} 
\usepackage{times}
\usepackage{helvet}
\usepackage{courier}

\usepackage{url}            
\usepackage{booktabs}       
\usepackage{amsfonts}       
\usepackage{nicefrac}       
\usepackage{microtype}      
\usepackage{csquotes}
\usepackage{graphicx}
\usepackage{float}

\title{Thompson Sampling For Stochastic Bandits with Graph Feedback}

\author{Aristide C. Y. Tossou \\ Computer Science and Engineering \\ Chalmers University of Technology \\ Gothenburg, Sweden\\ aristide@chalmers.se
  \And
  Christos Dimitrakakis\\  
  University of Lille, France\\ Chalmers University of Technology\\  Harvard University, USA\\ christos.dimitrakakis@gmail.com
  \And
  Devdatt Dubhashi \\ Computer Science and Engineering \\ Chalmers University of Technology \\ Gothenburg, Sweden\\ dubhashi@chalmers.se
}

\clearpage{}

\DeclareSymbolFont{bbold}{U}{bbold}{m}{n}
\DeclareSymbolFontAlphabet{\mathbbold}{bbold}
\usepackage{enumerate}
\usepackage{amsmath}
\usepackage{amsfonts}
\usepackage{amssymb}
\usepackage{amsbsy}
\usepackage{bbm}
\usepackage{isomath}
\usepackage{amsthm}
\usepackage{dsfont}
\usepackage{algorithm}
\usepackage{algpseudocode}
\usepackage{mathrsfs}
\usepackage{paralist}
\usepackage{epsfig}
\usepackage[caption=false]{subfig}
\usepackage{makeidx} 
\usepackage{mathtools}
\usepackage{comment}
\usepackage{breqn}
\usepackage{pgf}
\usepackage{tikz}

\numberwithin{equation}{section}

\usetikzlibrary{external}

\theoremstyle{plain}
\newtheorem{corollary}{Corollary}
\newtheorem{lemma}{Lemma}
\newtheorem{theorem}{Theorem}

\theoremstyle{definition}
\newtheorem{definition}{Definition}

\theoremstyle{remark}

\newtheorem{remark}{Remark}

\numberwithin{corollary}{section}
\numberwithin{lemma}{section}
\numberwithin{theorem}{section}
\numberwithin{assumption}{section}
\numberwithin{fact}{section}
\numberwithin{definition}{section}
\numberwithin{example}{section}
\numberwithin{conjecture}{section}
\numberwithin{remark}{section}
\numberwithin{claim}{section}

\usepackage{etoolbox}

\newcounter{numrel}
\renewcommand{\thenumrel}{\roman{numrel}}

\makeatletter
\newcommand{\numrel}[2]{
	\refstepcounter{numrel}
	\ltx@label{#2}
	\overset{(\thenumrel)}{#1}
}
\makeatother
\AfterEndEnvironment{align*}{\setcounter{numrel}{0}}

\newcommand{\algparbox}[1]{\parbox[t]{\dimexpr\linewidth-\algorithmicindent}{#1‌\strut}}

\newcommand \E {\mathop{\mbox{\ensuremath{\mathbb{E}}}}\nolimits}

\renewcommand \Pr {\mathop{\mbox{\ensuremath{\mathbb{P}}}}\nolimits}

\newcommand{\cset}[2]{\left\{\, #1 ~\middle|~ #2 \,\right\} }

\newcommand\Reals {{\mathds{R}}}

\newcommand \CP {{\mathcal{P}}}

\newcommand \CY {{\mathcal{Y}}}
\newcommand \CZ {{\mathcal{Z}}}

\newcommand \defn {\mathrel{\triangleq}}

\newcommand \argmax{\mathop{\rm arg\,max}}

\newcommand \dd{\,\mathrm{d}}

\DeclareMathAlphabet{\mathpzc}{OT1}{pzc}{m}{it}

\newcommand \Bernoulli {\mathop{\mathpzc{Bernoulli}}\nolimits}
\newcommand \Beta {\mathop{\mathpzc{Beta}}\nolimits}

\newcommand \Params {\Theta}
\newcommand \param {\theta}

\usepackage{tikz}
\usetikzlibrary{automata,topaths,shapes,arrows,fit,decorations.markings}
\tikzstyle{utility}=[diamond,draw=black,draw=blue!50,fill=blue!10,inner sep=0mm, minimum size=8mm]
\tikzstyle{select}=[rectangle,draw=black,draw=blue!50,fill=blue!10,inner sep=0mm, minimum size=6mm]
\tikzstyle{hidden}=[dashed,draw=black]
\tikzstyle{RV}=[circle,draw=black,draw=blue!50,fill=blue!10,inner sep=0mm, minimum size=6mm]

\def\clap#1{\hbox to 0pt{\hss#1\hss}}

\makeatletter
\let\OldStatex\Statex
\renewcommand{\Statex}[1][3]{
	\setlength\@tempdima{\algorithmicindent}
	\OldStatex\hskip\dimexpr#1\@tempdima\relax}
\makeatother

\clearpage{}
\newtheorem{proposition}[theorem]{Proposition}

\newcommand\NNODES{{500}}

\newcommand\NEDGESA{{2500}} 
\newcommand\NEDGESB{{62625}} 
\newcommand{\NTRIALS}{{\ensuremath{210}}}
\newcommand{\NGROUPS}{{\ensuremath{14}}}
\newcommand{\CONFIDENCE}{{\ensuremath{0.955}}}
\newcommand\erdos {{Erdős–Rényi}} 
\newcommand\UCBN{\textsf{UCB-N}} 
\newcommand\UCBMAX{\textsf{UCB-MaxN}}
\newcommand\TSN{\textsf{TS-N}}
\newcommand\TSMAXN{\textsf{TS-MaxN}} 
\newcommand\TS{Thompson Sampling} 
 
\newcommand\EPSGREEDYLP{$\epsilon$-\textsf{greedy-LP}}
\newcommand\EPSGREEDYD{$\epsilon$-\textsf{greedy-$\mathcal{D}$}}

\newcommand \Neighbor {\mathcal{N}}
\newcommand{\aGENT}{{agent}}
\newcommand{\aRM}{{arm}}

\newcommand{\aRMS}{{arms}}

\newcommand{\rOUND}{{round}}
\newcommand{\rOUNDS}{{rounds}}
\newcommand{\gAIN}{{reward}}
\newcommand{\gAINS}{{rewards}}
\newcommand{\NARM}{{K}}
\newcommand{\HORIZON}{\ensuremath{T}}

\newcommand \ACT {\ensuremath{\mathcal{V}}}
\newcommand \EDGES {\ensuremath{\mathcal{E}}}
\newcommand\REGRET {\ensuremath{\mathcal{L}}}
\newcommand\POLICY {\ensuremath{\pi}}
\newcommand\GRAPH{\ensuremath{G}}

\newcommand\CLIQUE{\ensuremath{\mathcal{C}}}

\newcommand {\INFO} {\mathbb{I}}
\newcommand {\ENTROPY} {\mathbb{H}}
\newcommand {\eqva} {\overline{a}}

\begin{document}

\maketitle

\begin{abstract}
We present a novel extension of Thompson Sampling for stochastic
sequential decision problems with graph feedback, even when the graph
structure itself is unknown and/or changing. We provide theoretical
guarantees on the Bayesian regret of the algorithm, linking its
performance to the underlying properties of the graph. Thompson
Sampling has the advantage of being applicable without the
need to construct complicated upper confidence bounds for different
problems. We illustrate its performance through extensive experimental
results on real and simulated networks with graph feedback. More
specifically, we tested our algorithms on power law, planted
partitions and \erdos{} graphs, as well as on graphs derived from
Facebook and Flixster data. These all show that our algorithms clearly
outperform related methods that employ upper confidence bounds, even
if the latter use more information about the graph.
 \end{abstract}

\section{Introduction}
Sequential decision making problems under uncertainty appear in most
modern applications, such as automated experimental design,
recommendation systems and optimisation. The common structure of these
applications that, at each time step $t$, the decision-making agent is
faced with a choice. After each decision, it obtains some
problem-dependent feedback~\cite{CesaBianchi-Lugosi:PLG}. For the
so-called \emph{bandit} problem, the choices are between different
\emph{arms}, and the feedback consists of a single scalar reward
obtained by the arm at time $t$. For the \emph{prediction} (or
full-information) problem, it obtains the reward of the chosen arm,
but also observes the rewards of all other choices at time $t$.  In
both cases, the problem is to maximise the total reward obtained over
time. However, dealing with specific types of feedback may require
specialised algorithms. In this paper, we show that the \emph{Thompson
  sampling} algorithm can be applied successfully to a range of
sequential decision problems, whose feedback structure is characterised
by a graph.

Our algorithm is an extension of Thompson sampling, introduced
in \cite{thompson1933lou}. Although easy to implement and effective in
practice, it remained unpopular until relatively recently. Interest
grew after empirical studies
\cite{scott2010modern,chapelle2011empirical} demonstrated performance
exceeding state of the art.  This has prompted a surge of interest in
Thompson sampling, with the first theoretical
results~\cite{agrawal:thompson} and industrial
adoption~\cite{scott2015multi} appearing only recently.
However, there are still only a few theoretical results and many of these are in the simplest settings. However, it is easy to implement and effective under very many different settings with complex feedback structures, and there is thus great need to extend the theoretical results to these wider settings.

\citeauthor{russo2014information} argue that Thompson sampling is a
very effective and versatile strategy for different \emph{information
  structures}.  Their paper focuses on specific examples: the two
extreme cases of no and full information mentioned above and the
case of linear bandits and combinatorial feedback.

Here we consider the case where the feedback is defined through a
graph~\cite{caron12,alon2015online}. More
specifically, the arms (choices) are vertices of a (potentially
changing) graph and when an arm is chosen, we see the reward of that
arm as well as its neighbours. On one hand, it is a clean model for
theoretical and experimental analysis and on the other hand, it also
corresponds to realistic settings in social networks, for example in
advertisement settings (c.f. \cite{caron12}).

We provide a problem-independent\footnote{In the sense that it does
  not depend on the reward structure.} regret bound that is
parametrized by the clique cover number of the graph and naturally
generalizes the two extreme cases of zero and full information. We
present two variants of Thompson sampling, that are both very easy to
implement and computationally efficient. The first is straightforward
Thompson sampling, and so draws an \aRM{} according to its probability
of being the best, but also uses the graph feedback to update the
posterior distribution. The second one can be seen as sampling cliques
in the graph according to their probability of containing the best
\aRM{}, and then choosing the empirically best \aRM{} in that clique.
Neither algorithm requires knowledge of the complete graph.

Almost all previous algorithms require the full structure of the
feedback graph in order to operate. Some require the entire graph for
performing their updates only at the end of round
(e.g. \cite{alon2015online})
Others actually need the description of the graph at the beginning of
the round to make their decision 
and almost none of the algorithms previously proposed in the
literature is able to provide non-trivial regret guarantees without
the feedback graphs being disclosed. However, \citeauthor{cohen2016online}  (\citeyear{cohen2016online}) argue that the assumption that the entire
observation system is revealed to the learner on each round, even if
only after making her prediction, is rather unnatural. In principle,
the learner need not be even aware of the fact that there is a graph
underlying the feedback model; the feedback graph is merely a
technical notion for us to specify a set of observations for each of
the possible arms. Ideally, the only signal we would like the agent to
receive following each round is the set of observations that
corresponds to the arm she has taken on that round (in addition to the
obtained reward). Our algorithms work in this setup - they do not need
the whole graph to be disclosed either when selecting the arm or when
updating beliefs - only the local neighborhood is needed. Furthermore,
the underlying graph is allowed to change arbitrarily at each step.
The detailed proofs of all our main results are available in the full version of this paper.

\section{Setting}

\subsection{The stochastic bandit model}
The stochastic \NARM-armed bandit problem is a well known sequential decision problem involving an \aGENT{} sequentially choosing among a set of \NARM{} \aRMS{} $\ACT{} = \{ 1 \ldots \NARM{} \}$. At each \rOUND{} $t$, the \aGENT{} plays an \aRM{} $A_t \in \ACT{}$ and receives a \gAIN{} $r_t = R(Y_{t, A_t})$, where $Y_{t,A_t} : \Omega \to \CY $  is a random variable defined on some probability space $(P, \Omega, \Sigma)$ and $R : \CY \to \Reals$ is a \gAIN{} function.

Each  \aRM{} $i$ has mean reward $\mu_i(P) = \E_P R(Y_{t, i})$.
 Our goal is to maximize its expected cumulative \gAIN{} after \HORIZON{} \rOUNDS{}. An equivalent notion is to minimize the expected regret against an oracle which knows $P$. More formally, the expected regret $\E^\POLICY_P\REGRET{}$ of an agent policy $\POLICY$ for a bandit problem $P$ is defined as:
\begin{align}
	\E^\POLICY_P\REGRET{} = \HORIZON\mu_{*}(P) -  \E_P^{\POLICY}\sum_{t=1}^{\HORIZON{}} r_{A_t},
\end{align} 
where $\mu_{*}(P) = \max_{i \in \ACT} \mu_i(P)$ is the mean of the optimal \aRM{} and $\POLICY(A_t | h_t)$ is the policy of the \aGENT{}, defining a probability distribution on the next \aRM{} $A_t$  given the history $h_t = \langle A_{1:t-1}, r_{1:t-1} \rangle$ of previous \aRMS{} and \gAINS{}.

The main challenge in this model is that the \aGENT{} does not know $P$, and it only observes the reward of the \aRM{} it played. As a consequence, the \aGENT{} must trade-off exploitation (taking the apparently  best \aRM{}) with exploration (trying out other \aRMS{} to assess their quality). 

The Bayesian setting offers a natural way to model this uncertainty, by assuming that the underlying probability law $P$ is in some set $\CP = \cset{P_\param}{\param \in \Params}$ parametrised by $\param$, over which we define a prior probability distribution $\Pr$. In that case, we can define the Bayesian regret:
\begin{equation}
  \label{eq:Bayesian-regret}
  \E^\POLICY \REGRET{}  = \int_\Params \E^\POLICY_{P_\param} (\REGRET{}) \dd \Pr(\param) .
\end{equation}
A policy with small Bayesian regret may not be uniformly good in all $P$. 
Since in the Bayesian setting we frequently need to discuss posterior probabilities and expectations, we also introduce the notation $\E_t f \defn \E (f \mid h_t)$ and $\Pr_t(\cdot) \defn \Pr(\cdot \mid h_t)$ for expectations and probabilities conditioned on the current history.

\subsection{The graph feedback model}
In this model, we assume the existence of an undirected graph $\GRAPH{} = (\ACT,\EDGES)$ with vertices corresponding to arms. By taking an arm $a \in \ACT$, we not only receive the reward of the arm we played, but we also observe the rewards of all neighbouring arms $\Neighbor{}_a= \cset{a' \in \ACT}{(a,a') \in E}$.
More precisely, at each time-step $t$ we observe $Y_{t, a'}$ for all $a' \in \Neighbor{}_{A_t}$, while our reward is still $r_t = R(Y_{t, A_t})$.

If the graph is empty, then the setting is equivalent to the bandit problem. If the graph is fully connected, then it is equivalent to the prediction (i.e. full information) problem. However, many practical graphs, such as those derived from social networks, have an intermediate connectivity. In such cases, the amount of information that we can obtain by picking an arm can be characterised by graph properties, such as the clique cover number:
\begin{definition}[Clique cover number]
A clique covering $\CLIQUE$ of a graph
 \GRAPH{} is a partition of all its vertices  into sets $S \in \CLIQUE{}$ such that the sub-graph formed by each $S$ is a clique i.e. all vertices in $S$ are connected to each other in \GRAPH{}. The smallest number of cliques into which the nodes of \GRAPH{} can be partitioned is called the clique cover number.  We denote by $\CLIQUE{}(\GRAPH)$ the minimum clique cover and $\chi(\overline{\GRAPH})$ its size, omitting $\GRAPH$ when clear from the context. 
\end{definition}
The domination number is another useful similar notion for the amount of information that we can obtain.
\begin{definition}[Domination number]
A dominating set in a graph $\GRAPH = (\ACT,\EDGES)$ is a subset $U \subseteq \ACT$ such that for every vertex $u \in V$, either  $u \in U$ or $(u,v) \in \EDGES$ for some $v \in U$. The smallest size of a dominating set in $G$ is called the domination number of $\GRAPH$ and denoted $\gamma(\GRAPH)$.  
\end{definition}

\section{Related work and our contribution}
Optimal policies for the stochastic multi-armed bandit problem were
first characterised by~\cite{lai1985asymptotically}, while index-based
optimal policies for general non-parametric problems were given
by~\cite{burnetas1997optimal}. Later \cite{finitetimemab} proved
finite-time regret bounds for a number of UCB (Upper Confidence Bound)
index policies, while \cite{garivier2011kl} proved finite-time bounds
for index policies similar to those of \cite{burnetas1997optimal},
with problem-dependent bounds $O(\NARM \ln T)$. Recently, a number of
policies based on sampling from the posterior distribution
(i.e. Thompson sampling\cite{thompson1933lou}) were analysed in both
the frequentist~\cite{agrawal:thompson} and Bayesian
setting~\cite{russo2014information} and shown to obtain the same order
of regret bound for the stochastic case.  For the \emph{adversarial}
bandit problem the bounds are of order $O(\sqrt{\NARM T})$.  The
analysis for the full information case generally results in
$O(\ln (\NARM) \sqrt{T})$ bounds on the
regret~\cite{CesaBianchi-Lugosi:PLG}, i.e. with a much lower
dependence on the number of arms.

Intermediate cases between full information and bandit feedback can be
obtained through graph feedback, introduced in
\cite{mannor2011bandits}, which is the focus of this paper.  In
particular, \cite{caron12} and \cite{alon2015online} analysed graph
feedback problems with stochastic and adversarial reward sequences
respectively. Specifically, \citeauthor{caron12} analysed variants of Upper
Confidence Bound policies, for which they obtained
$O(\chi(\overline{G})\ln T)$ problem-dependent bounds. In more recent work,
\cite{cohen2016online} also introduced algorithms for graphs where the
structure is never fully revealed showing that (unlike the bandit
setting) there is a large gap in the regret between the adversarial
and stochastic cases. In particular, they show that in the adversarial
setting one cannot do any better than ignore all additional feedback,
while they provide an action-elimination algorithm for the stochastic
setting.
Finally, \cite{buccapatnam2014stochastic} obtain a problem-dependent bound of
the form  $O(\gamma^*(G) \log T + K \delta)$
where $\gamma^*$ is the linear programming relaxation to $\gamma$ and $\delta$ is the minimum degree of $G$.

\paragraph{Contributions.} In this paper, we provide much simpler strategies based on Thompson sampling, with a matching regret bound. Unlike previous work, these are also applicable to graphs whose structure is unknown or changing over time. More specifically:
\begin{enumerate}
\item We extend~\cite{russo2014information} to graph-structured feedback, and obtain a problem-independent bound of $O(\sqrt{\frac{1}{2} \chi(\overline{G}) T})$.
\item Using planted partition models, we verify the bound's dependence on the clique cover.
\item We provide experiments on data drawn from two types of random graphs: \erdos{} graphs and power law graphs, showing that our algorithms clearly outperform UCB and its variations~\cite{caron12}.
\item Finally, we measured the performance on graphs estimated from the data used in \cite{caron12}. Once again, Thompson sampling clearly outperforms UCB and its variants.
\end{enumerate}

\section{Algorithms and analysis}
We consider two algorithms based on Thompson sampling. The first uses standard Thompson sampling to select arms. As this also reveals the rewards of neighbouring arms, the posterior is conditioned on those as well. The second algorithm uses Thompson sampling to select an arm, and then chooses the empirically best arm within that arm's clique.

\subsection{The \TSN{} policy}
The \TSN{} policy is an adaptation of \TS{} for graph-structured
feedback. \TS{} maintains a distribution over the problem parameters.
At each step, it selects an arm according to the probability of its
mean being the largest. It then observes a set of rewards which it
uses to update its probability distribution over the parameters.

For the case where each arm has an independent parameter defining its reward distribution, we can update the distribution of all arms observed separately. A particularly simple case is when all the reward are generated from Bernoulli distributions. Then we can simply use a Beta prior for each arm, illustrated by the  \TSN{} policy in Algorithm~\ref{alg:tsn}. We note that the algorithm trivially extends to other priors and families.
\begin{algorithm}                      
	\caption{\TSN{} (Bernoulli case)}
	\label{alg:tsn}                           
	\begin{algorithmic}
		
		\State For each arm $i$, set $S_i = 1$ and $F_i = 1$
		
		\ForAll{ round $t=1,\cdots, T$}
		
			\State \algparbox{For each arm $i$, sample $\theta_i$ from the Beta distribution $\Beta(S_i, F_i)$}
			\State Play arm $A_t = \argmax_{i \in \ACT{}} \theta_i$ 
			
			\ForAll{$k \in \Neighbor_{A_t}$}
			
				\State $\hat{r}_k = \Bernoulli(r_k)$
				
				\State If $\hat{r}_k = 1$ the $S_k = S_k + 1$, else $F_k = F_k + 1$
			
			\EndFor
		
		\EndFor
		
	\end{algorithmic}
	
\end{algorithm}

\subsection{The \TSMAXN{} policy}
The \TSN{} policy does not fully exploit the graphical structure. For example, as noted by \cite{caron12}, instead of doing exploration on \aRM{} $i$ we could explore an apparently better neighbour, which would give us the same information. More precisely, instead of picking arm $i$, we pick the arm $j \in \Neighbor_i$ with the best empirical mean. The intuition behind it is that, if we take any \aRM{} in $\Neighbor_i$, we are going to observe anyway the reward of $i$. So it is always better to exploit the best arm in $\Neighbor_i$.
The resulting policy, \TSMAXN{} is summarized in Algorithm \ref{alg:tsmaxn}. Although our theoretical results do not apply to this policy, it can have better performance as it uses more information.
\begin{algorithm}                      
	\caption{\TSMAXN{}}
	\label{alg:tsmaxn}                           
	\begin{algorithmic}
		
		\State For each arm $i$, set $S_i = 1$ and $F_i = 1$
		\State Let $\bar{x}_i$ be the empirical mean of arm $i$
		
		\ForAll{ round $t=1,\cdots, T$}
		
		\State \algparbox{For each arm $i$, sample $\theta_i$ from the Beta distribution $\Beta(S_i, F_i)$}
		\State Let $j =  \argmax_{i \in \ACT{}} \theta_i$  
		\State Play arm $A_t = \argmax_{k \in \Neighbor_{j}} \bar{x}_k$ 
		
		\ForAll{$k \in \Neighbor_{A_t}$}
		
		\State $\hat{r}_k = \Bernoulli(r_k)$
		
		\State If $\hat{r}_k = 1$ the $S_k = S_k + 1$, else $F_k = F_k + 1$
		
		\EndFor
		
		\EndFor
		
	\end{algorithmic}
	
\end{algorithm}

\subsection{Analysis of \TSN{} policy}

Russo and van Roy introduced an elegant approach to the analysis of
Thompson sampling. They define the \emph{information ratio} as a key
quantity for analysing information structures:
\begin{equation}
\label{eq:ir}
\Gamma_t : = \frac{\E_t \left[ R(Y_{t,A^*}) - R(Y_{t,A_t})\right]^2}{\INFO_t(A^*, (A_t,Y_{t,A_t}))},
\end{equation}
where $\E_t$ and $\INFO_t$ denote expectation and mutual information respectively, conditioned on the history of arms and observations until time $t$.
They show that it follows very generally that
\begin{proposition}
\label{prop:ts}
If $\Gamma_t \leq \Gamma$ almost surely for all $1 \leq t \leq T$, then,
$\E \REGRET(T, \pi^{TS}) \leq \sqrt{\Gamma \ENTROPY(\alpha_1) T}.$
\end{proposition}
Here $\ENTROPY$ denotes entropy.
Thus to analyse the performance of Thompson sampling on a specific problem, one may focus on bounding the 
information ratio (\ref{eq:ir}). For the (independent) $K$-armed bandit
case, they show that $\Gamma_t \leq K/2$, while 
for full-information ($K$ experts) case, they show
that $\Gamma_t \leq 1/2$. We now give a simple but
useful extension of their results which is intermediate between these
cases.

\begin{proposition}
\label{prop:eq}
Let $\equiv$ be an equivalence relation defined on the arms with
$\eqva$ denoting the equivalence class of $a$.  Let
$Y_{t,a} = (a, Z_{t,\eqva})$ for sequence of random variables
$Z_{t, \eqva} : \Omega \to \CZ$.  Then
$\Gamma_t \leq \frac{1}{2} |K/\equiv|$, half the number of equivalence
classes.
\end{proposition}
This is a direct generalisation of propositions 3 and 4 in \cite{russo2014information}, to which it reduces when the equivalence relation is trivial (bandit case) or full (expert case).

We can now use Proposition~\ref{prop:eq} to analyse
graph structured arms:
\begin{lemma}
\label{lem:gr}
Let $\GRAPH=(\ACT,\EDGES)$ be a graph with $V$ corresponding to the arms and
suppose that when an arm $a$ is played, we observe the rewards 
$R(Y_{t,a'})$ for all $a' \in N(a)$ i.e we observe the rewards
corresponding to both $a$ and all its neighbours. Let ${\cal C}$ be a
clique cover of $G$ i.e. a partition of $V$ into cliques. Then
$\Gamma_t \leq \frac{1}{2} |{\cal C}|$.
\end{lemma}

Applying Proposition~\ref{prop:ts} and Lemma~\ref{lem:gr}, we get a
performance guarantee for Thompson sampling with graph-structured feedback:

\begin{theorem}
\label{th:gr}
For Thompson sampling with feedback from the graph
$G$, we have $\E^{\pi^{TS}} \REGRET \leq \sqrt{\frac{1}{2} \chi(\overline{G}) \ENTROPY(\alpha_1)
  T}$, where $\chi(\overline{G}) $ is the \emph{clique cover number} of $G$.
\end{theorem}

\begin{remark}
The bandit and expert cases are special cases corresponding to the
empty graph and the complete graph respectively since
$\chi(\overline{G}) = K$ for the empty graph and $\chi(\overline{G}) =
1$ for the complete graph.
\end{remark}

\begin{remark}[Planted Partition Models]
  The \emph{planted partition models} or \emph{stochastic block
    models} graphs $G(n,k,p,q)$ are defined as follows
  \cite{mcsherry01,condon01}: first a fixed partition of the $n$
  vertices into $k$ parts is chosen, then an edge between two vertices
  within the same class exists with probability $p$ and that between
  vertices in different classes exists with probability $q$,
  independently with $p > q$. If $p=1$, then with high probability,
  the clique cover number of the resulting graph is $k$ (corresponding
  to the planted $k$ cliques). Thus for this class of graphs, the
  regret grows as $O(\sqrt{k})$ as per Theorem~\ref{th:gr}. This is
  explored in Section~\ref{sec:experiments}. When $p \not = 1$ but large, the planted
  partition graph is considered a good model of the structure of
  network communities.
\end{remark}

If the underlying graph changes at each time step, then we also have
the bound for the same algorithm:
\begin{corollary}
	Suppose the underlying graph at time $t \geq 1 $ is $G_t$, then:
	\[ \E^{\pi^{TS}} \REGRET \leq \sqrt{\frac{1}{2} \max_t \chi(\overline{G_t}) \ENTROPY(\alpha_1)
		T} \]
\end{corollary}
\begin{proof}
  The information ratio at time $t$ is bounded by $\chi(\overline{G_t}) \leq \max_t  \chi(\overline{G_t})$.
\end{proof}

\section{Experiments}
\label{sec:experiments}
We compared our proposed algorithms in terms of the actual expected regret against a number of other algorithms that can take advantage of the graph structure. Our comparison was performed over both synthetic graphs and networks derived from real-world data.\footnote{Our source code and data sets will be made available on an open hosting website.}

\subsection{Algorithms and hyperparameters.}
In all our experiments, we tested against the \UCBMAX{} and \UCBN{} algorithms, introduced in ~\cite{caron12}. These are the analogues of our algorithms, using upper confidence bounds instead of Thompson sampling.

\subsubsection{\EPSGREEDYD{} and \EPSGREEDYLP{}.}

For the real-world networks, we also evaluated our algorithms against
a variant of \EPSGREEDYLP{} from \cite{buccapatnam2014stochastic}.
This is based on a linear program formulation for finding a lower
bound $\gamma(G)$ on the size of the minimum dominating set. We
observe first that their analysis holds for \emph{any} fixed
dominating set $D$ and the bound so obtained is $O(|D| \ln T)$. In
particular, we may use a simple greedy algorithm to compute a
near--optimal dominating set $D'$ such that
$|D'| \leq \gamma(G) \log \Delta$, where $\Delta$ is the maximum
degree of the graph. \footnote{No polynomial time algorithm can
  guarantee a better approximation unless \textsc{P}=\textsc{NP},
  \cite{ruan2004greedy}} Using such a near optimal dominating set in
place of the LP relaxation and choosing arms from it uniformly at
random, we obtain a variant of the original algorithm, which we call
\EPSGREEDYD{}, which is much more computationally efficient, and which
enjoys a similar regret bound:
\begin{theorem}
  The regret of \EPSGREEDYD{} is at most
  $O(\gamma(G) \ln \Delta \ln T)$, where $\Delta$ is the maximum degree
  of the graph.
\end{theorem}

\EPSGREEDYD{} and \EPSGREEDYLP{} have the hyper-parameters $c$ and $d$, which control the
amount of exploration. We found that its performance is highly
sensitive to their choice. In our experiments, we find the optimal
values for these parameters by performing a separate grid search for
each problem, and only reporting the best results. Since there is no
obvious way to tune these parameters online, this leads to a favourable bias in
our results for this algorithm.
\footnote{A similar observation was made in~\cite{auer2002finite}, which noted that an optimally tuned $\epsilon$-greedy performs almost always best, but its performance can degrade significantly when the parameters are changed. Although \cite{buccapatnam2014stochastic} suggests a method for selecting these parameters, we find that using it leads to a near-linear regret.}

As Thompson sampling is a Bayesian algorithm, we can view the prior
distribution as a hyper-parameter. In our experiments, we always set
that to a Beta(1,1) prior for all rewards.

\subsection{General experimental setup.}
For all of our experiments, we performed \NTRIALS{} independent trials and reported the \emph{median-of-means} estimator\footnote{Used heavily in the streaming literature~\cite{alon1996space}} of the cumulative regret. It partitions the trials into $a_0$ equal groups and return the median of the sample means of each group. 
We set the number of groups to $a_0 = \NGROUPS{}$, so that the confidence interval holds with probability at least $\CONFIDENCE{}$.

We also reported the deviation of each algorithm using the Gini's Mean Difference (GMD hereafter) \cite{gini1912variabilita}. GMD computes the deviation as $\sum_{j=1}^{N} (2j-N-1)x_{(j)}$ with $x_{(j)}$ the $j$-th order statistics of the sample (that is $x_{(1)} \leq x_{(2)} \leq \ldots \leq x_{(N)}$). As shown in \cite{yitzhaki2003gini,david1968miscellanea} the GMD provides a superior approximation of the true deviation than the standard one. To account for the fact that the cumulative regret of our algorithms might not follow a symmetric distribution, we computed the GMD separately for the values above and below the \emph{median-of-means}.

\subsection{Simulated graphs}

In our synthetic problems, unless otherwise stated, the rewards are drawn from a Bernoulli distribution whose mean is generated uniformly randomly in $[0.45,0.55]$ except for the optimal arm whose mean is generated randomly in $[0.55,0.6]$.  
The number of nodes in the graph is \NNODES{}. We tested with  a sparse graph of \NEDGESA{} edges and also with a dense graph of \NEDGESB{} edges.

\paragraph{Erdős–Rényi graphs}

In our first experiment, we generate the graph randomly using the \erdos{} model.
Figure~\ref{fig:erdos:small} and ~\ref{fig:erdos:large} respectively show the result in the sparse and dense graph.

Our first observation here is that all policies take advantage of a large number of edges as their cumulative regret is better by using the dense graph (Figure \ref{fig:erdos:large}) rather than the sparse one (Figure ~\ref{fig:erdos:small}). This confirms the theoretical result as a dense graph will have a smaller clique cover number than a sparse one.

The policy \TSMAXN{} outperforms all other in both the sparse and dense graph model. However, the performance of \TSN{} is very close to that of \TSMAXN{} in the near complete graph. This is explained by the fact that in a near complete graph we have many cliques. It is revealing to see that \TSN{} outperforms both the \UCBN{} and \UCBMAX{} policies.

\paragraph{Power Law graph}
Such graphs are commonly used to generate static scale-free networks \cite{goh2001universal}.
In this experiment, we generated a non-growing random graph with expected power-law degree distribution.

show the results respectively for the dense and sparse graph
Figure \ref{fig:powerlaw:large} and \ref{fig:powerlaw:small} show the results respectively for the dense and sparse graph. Again, the policy \TSMAXN{} clearly outperforms all other. In the sparse graph model, \TSN{} is beaten by \UCBMAX{} at the beginning of the rounds ( $t \leq 100000$), but catches and ended up beating \UCBMAX{}.

\paragraph{Planted Partition Model}

The aim of the experiment on this model is to check the dependency on the number of cliques for each policy. Figure \ref{fig:plantedpartition} shows the results where on the x-axis we have the parameter $k$ of the planted partition graph (which is almost equal to the number of cliques) on a graph with 1024 nodes. On the y-axis we have the relative regret of each policy, i.e. the ratio between the regret of each policy with the regret of the best policy when there are two groups, for ease of comparison. As we can see, all methods' regret scales similarly. Thus, the theoretical bounds appear to hold in practice, and to be somewhat pessimistic. For a larger number of nodes, we would expect the plots to flatten later.

\subsection{Social networks datasets}
Our experiments on real world datasets follow the methodology described in \cite{caron12}. We first infer a graph from data, and then define a reward function for movie recommendation from user ratings. Missing ratings are predicted using matrix factorization. This enables us to generate rewards from the graph.
We explain the datasets, reward function and graph inference in the full version.

\paragraph{Results}
Figure~\ref{fig:facebook:100} shows the results for the Facebook graph and
Figure~\ref{fig:flixster:100} for the Flixster graph. Once again, the Thompson sampling strategies dominate all
other strategies for the Facebook and they are matched by the
optimised \EPSGREEDYD{} policy in the Flixster graph. We notice that
in this setting the gap between the UCB policies and the rest is much
larger, as is the overall regret of all policies. This can be
attributed to the larger size of these graphs.

\section{Conclusion}
We have presented the first Thompson sampling algorithms for sequential decision problems with graph feedback, where we not only observe the reward of the arm we select, but also those of the neighbouring arms in the graph. Thus, the graph feedback allows us the flexibility to model different types of feedback information, from bandit feedback to expert feedback. 
Since the structure of the graph need not be known in advance, our algorithms are directly applicable to problems with changing and/or unknown topology. Our analysis leverages the information-theoretic construction of \cite{russo2014information}, by bounding the expected information gain in terms of fundamental graph properties. Although our problem-independent bound of is not directly comparable to \cite{caron12}, we believe that a problem-independent version of the latter should be $O(\sqrt{\chi \ln T})$, in which case our results would represent an improvement of $O(\sqrt{\chi})$.

In practice, our two variants always outperform \UCBN{}, \UCBMAX{}, which also use graph feedback but rely on upper confidence bounds. We are also favourably compared against \EPSGREEDYD{}, even when we tune the parameters of the latter \emph{post hoc}.

It would be interesting to extend our techniques to other types of
feedback.  For example, the Bayesian foundations of Thompson sampling render our
algorithms applicable to arbitrary dependencies between arms. In
future work, we will analytically and experimentally consider such
problems and related applications. Finally, an open question is the
existence of information-theoretic lower bounds in settings with
partial feedback.

\paragraph{Acknowledgments.} This research was partially supported by the SNSF grants ``Adaptive control with approximate Bayesian computation and differential privacy'' (IZK0Z2\_167522), ``Swiss Sense Synergy'' (CRSII2\_154458), by the the People Programme (Marie Curie Actions) of the European Union's Seventh Framework Programme (FP7/2007-2013) under REA grant agreement number 608743, and the Future of Life Institute.

\begin{figure}[H]
	\includegraphics[width=0.45\textwidth]{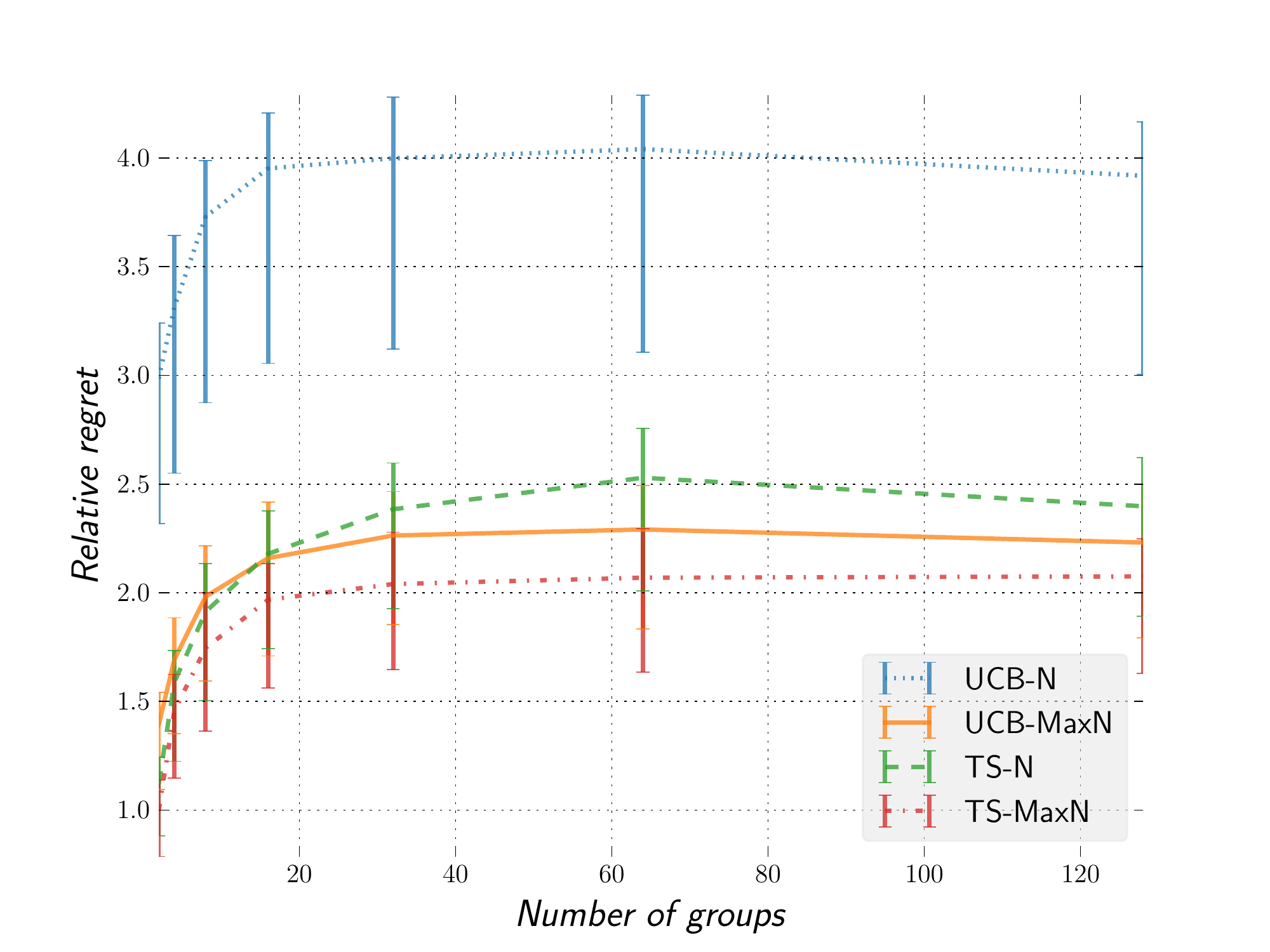}
	\caption{Relative regret in the planted partition setting as the number of groups increases.} 
	\label{fig:plantedpartition}
\end{figure}

\begin{figure*}
	\subfloat[Facebook]{
		\includegraphics[width=0.45\textwidth]{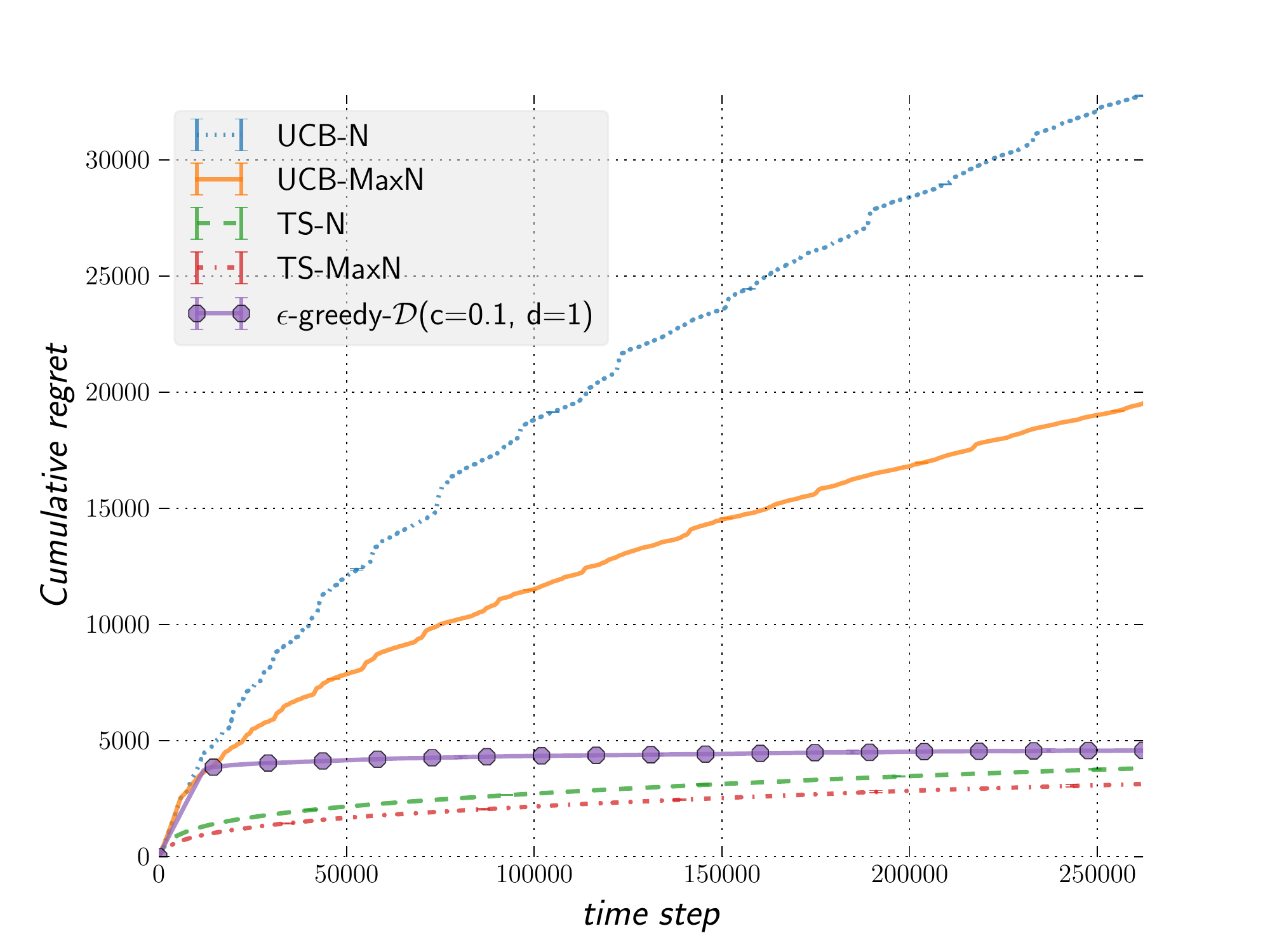}\label{fig:facebook:100}} \subfloat[Flixster]{
		\includegraphics[width=0.45\textwidth]{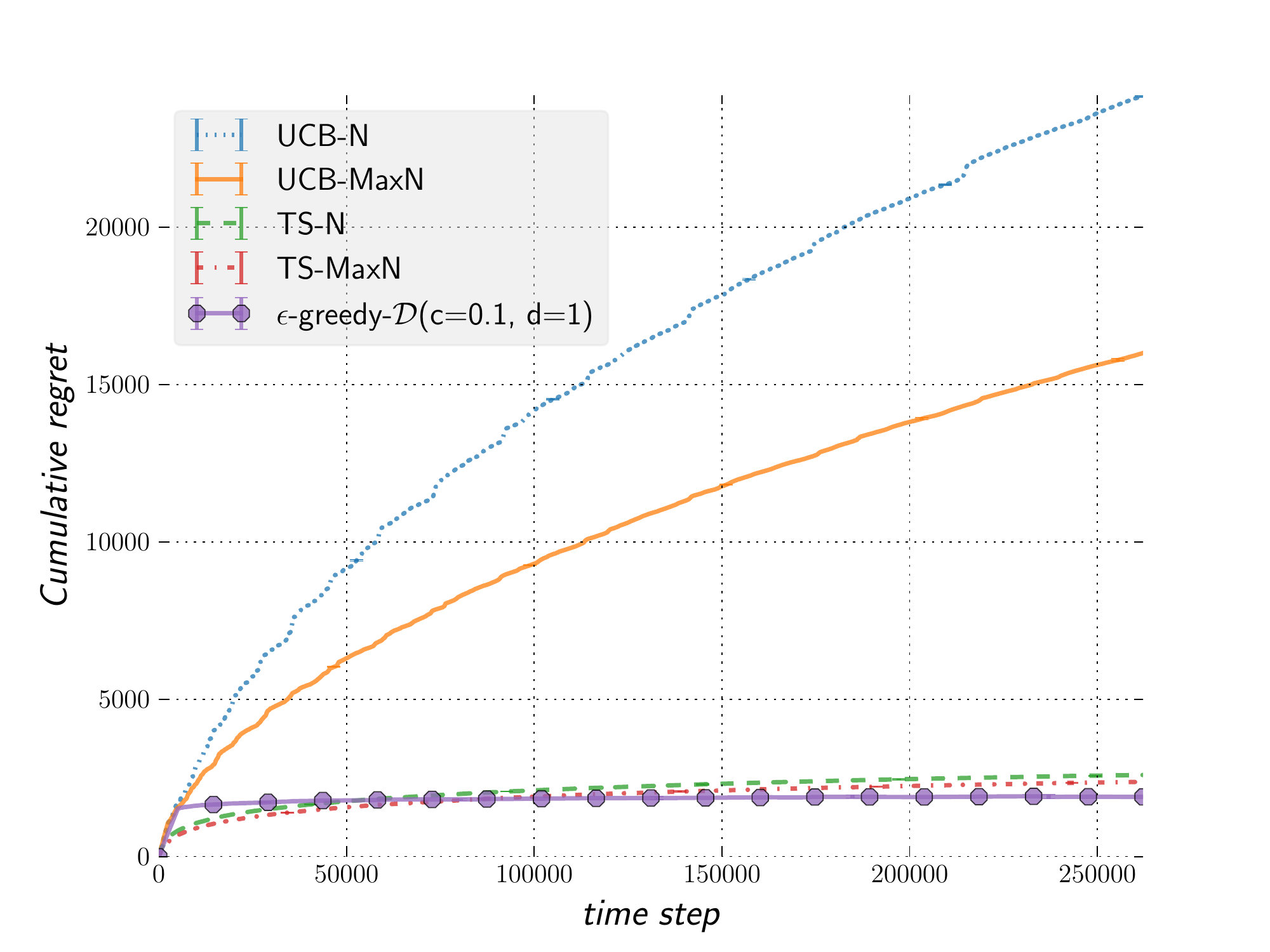}\label{fig:flixster:100}}\\ 
	\subfloat[Power law sparse]{
		\includegraphics[width=0.45\textwidth]{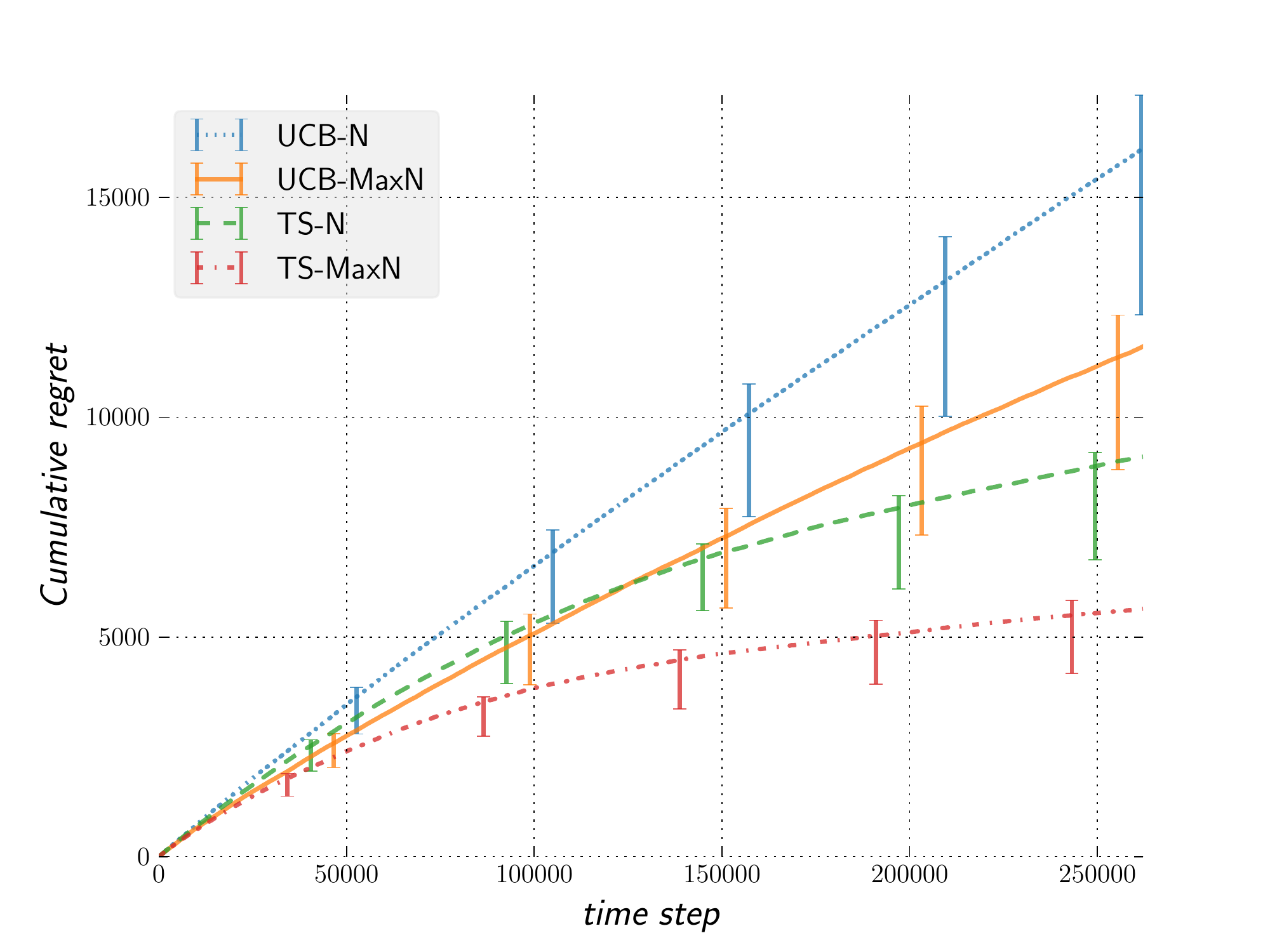}\label{fig:powerlaw:small}} \subfloat[Power law dense]{
		\includegraphics[width=0.45\textwidth]{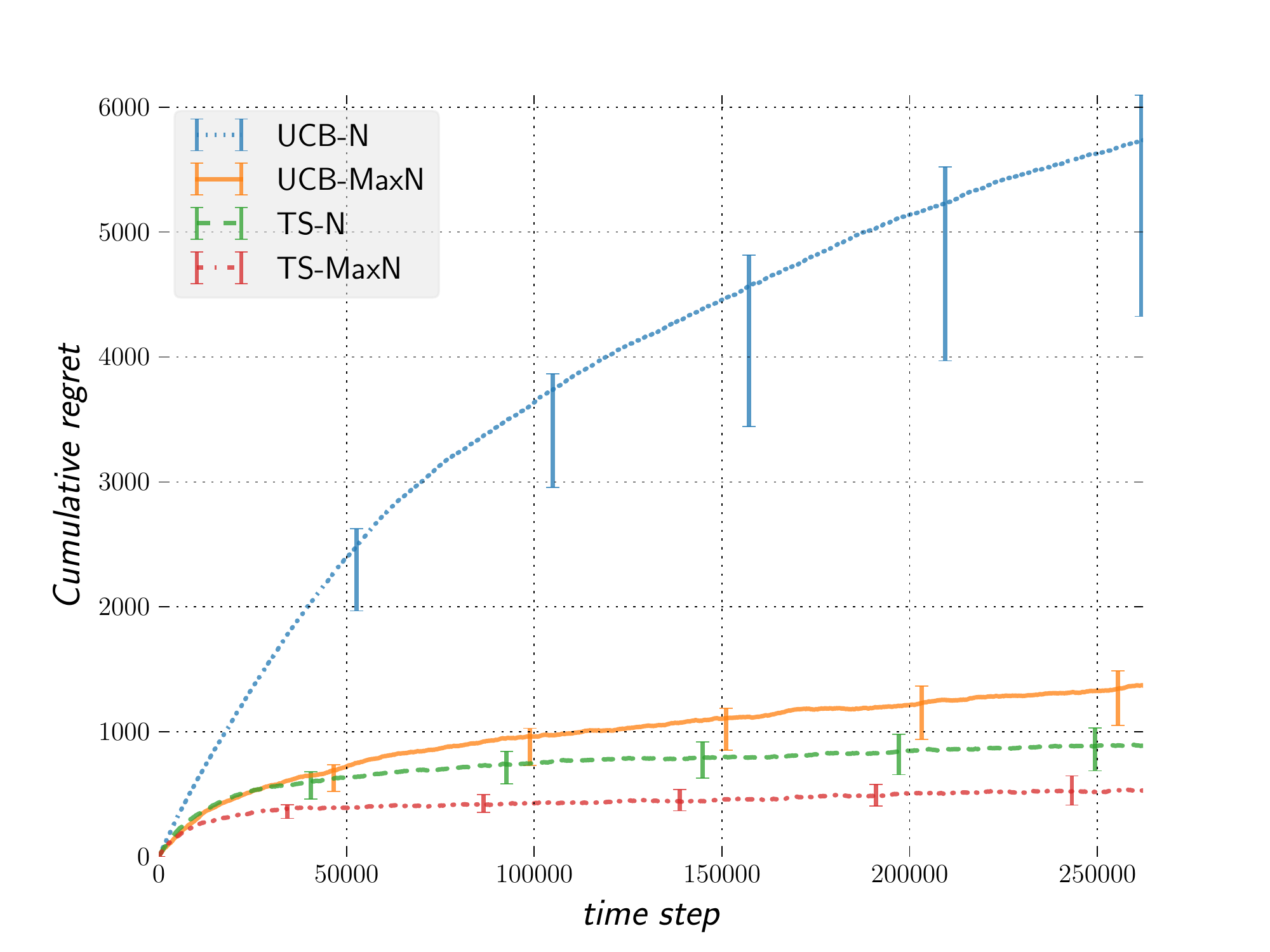}\label{fig:powerlaw:large}} \\ \subfloat[\erdos{} sparse]{
		\includegraphics[width=0.45\textwidth]{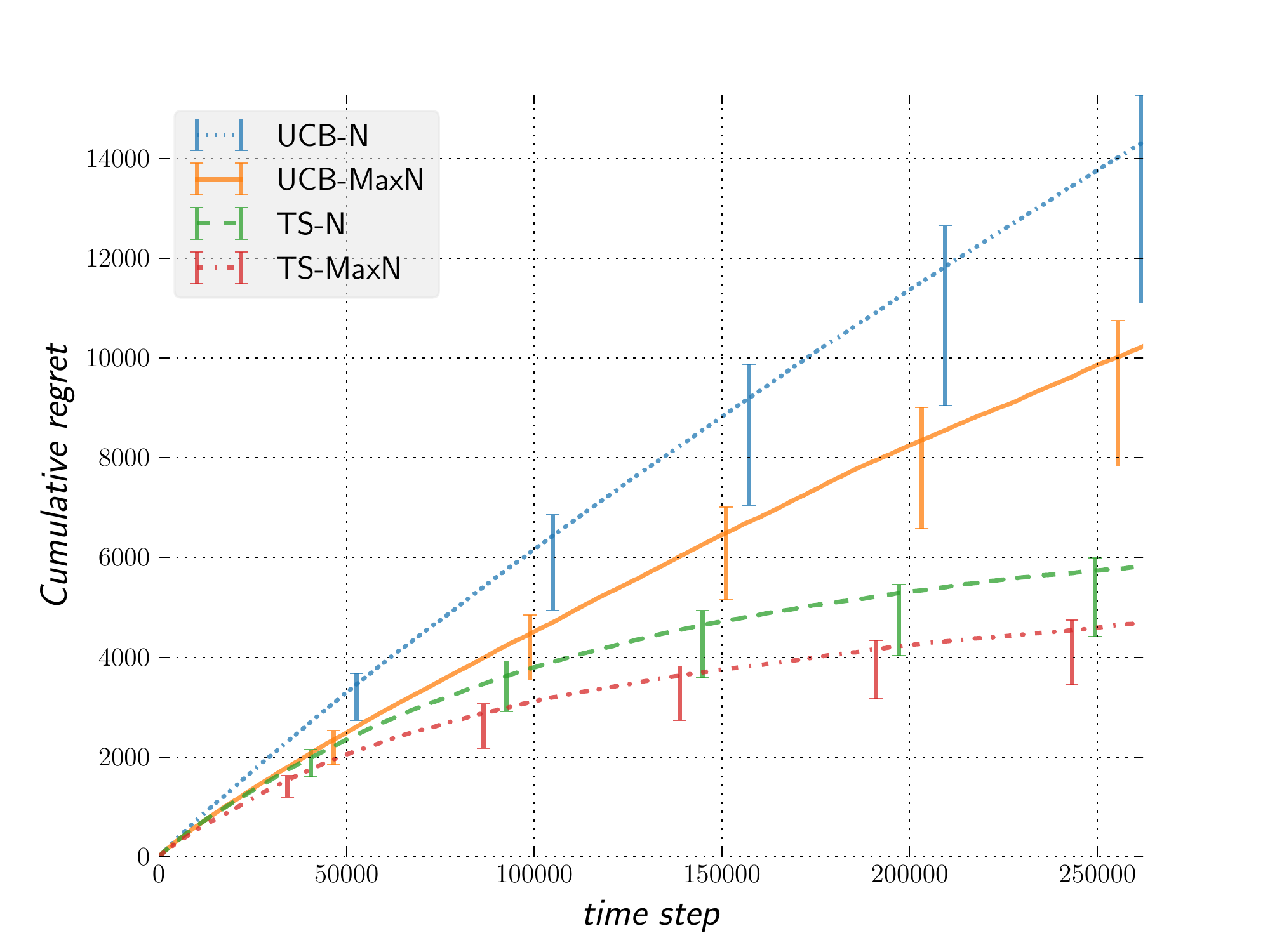}\label{fig:erdos:small}} \subfloat[\erdos{} Dense]{
		\includegraphics[width=0.45\textwidth]{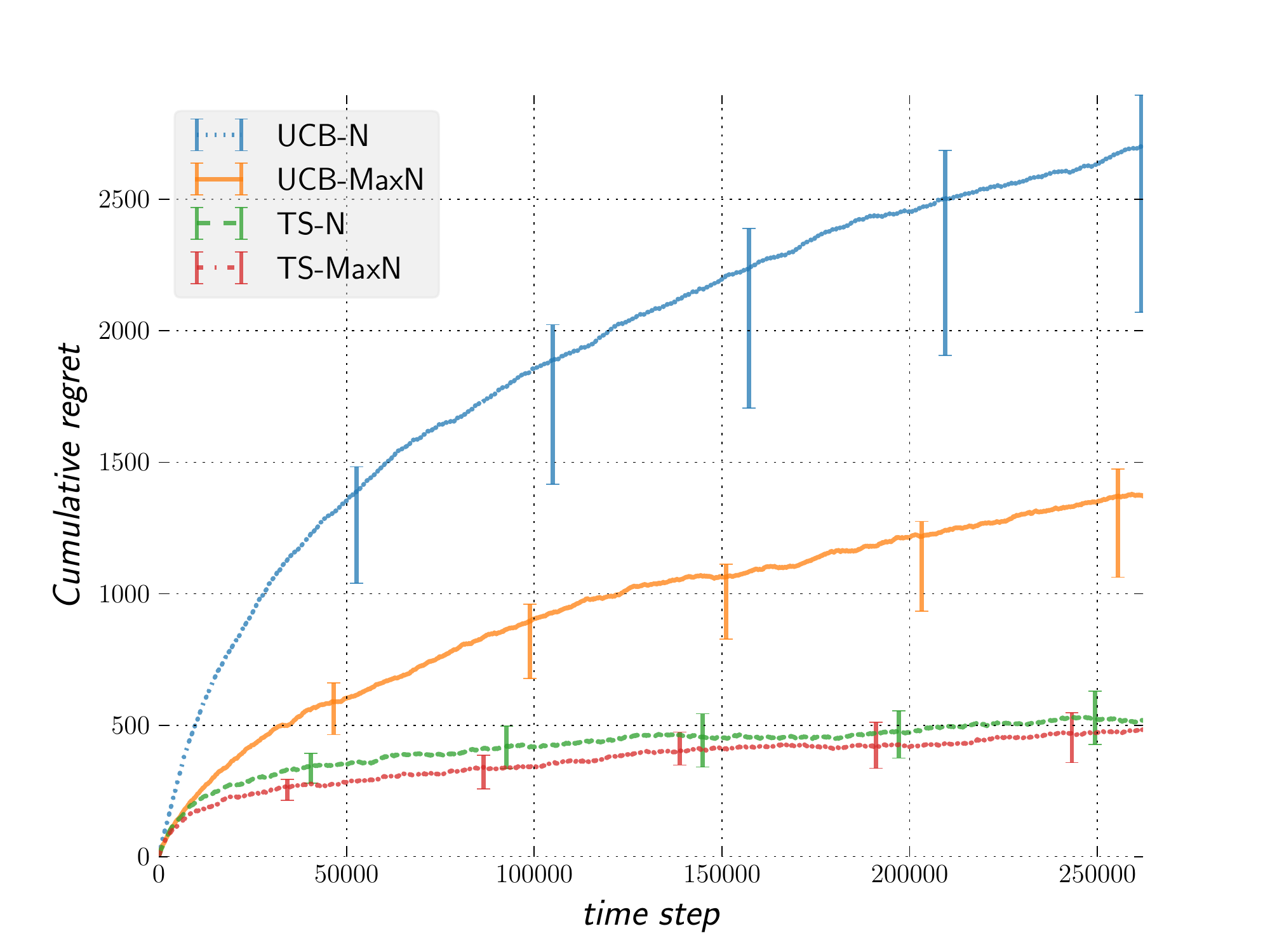}\label{fig:erdos:large}}
	\caption{Regret and error bar on simulated (sparse and dense)  and real social network graphs}\label{fig:experiments}
\end{figure*}


\begin{thebibliography}{}

\bibitem[\protect\citeauthoryear{Agrawal and Goyal}{2012}]{agrawal:thompson}
Agrawal, S., and Goyal, N.
\newblock 2012.
\newblock Analysis of thompson sampling for the multi-armed bandit problem.
\newblock In {\em {COLT 2012}}.

\bibitem[\protect\citeauthoryear{Alon \bgroup et al\mbox.\egroup
  }{2015}]{alon2015online}
Alon, N.; Cesa-Bianchi, N.; Dekel, O.; and Koren, T.
\newblock 2015.
\newblock Online learning with feedback graphs: beyond bandits.
\newblock In {\em Proceedings of the 28th Annual Conference on Learning
  Theory},  23--35.

\bibitem[\protect\citeauthoryear{Alon, Matias, and
  Szegedy}{1996}]{alon1996space}
Alon, N.; Matias, Y.; and Szegedy, M.
\newblock 1996.
\newblock The space complexity of approximating the frequency moments.
\newblock In {\em 28th STOC},  20--29.
\newblock ACM.

\bibitem[\protect\citeauthoryear{Auer, Cesa-Bianchi, and
  Fischer}{2002a}]{finitetimemab}
Auer, P.; Cesa-Bianchi, N.; and Fischer, P.
\newblock 2002a.
\newblock Finite time analysis of the multiarmed bandit problem.
\newblock {\em Machine Learning} 47(2/3):235--256.

\bibitem[\protect\citeauthoryear{Auer, Cesa-Bianchi, and
  Fischer}{2002b}]{auer2002finite}
Auer, P.; Cesa-Bianchi, N.; and Fischer, P.
\newblock 2002b.
\newblock Finite-time analysis of the multiarmed bandit problem.
\newblock {\em Machine learning} 47(2-3):235--256.

\bibitem[\protect\citeauthoryear{Buccapatnam, Eryilmaz, and
  Shroff}{2014}]{buccapatnam2014stochastic}
Buccapatnam, S.; Eryilmaz, A.; and Shroff, N.~B.
\newblock 2014.
\newblock Stochastic bandits with side observations on networks.
\newblock {\em ACM SIGMETRICS Performance Evaluation Review} 42(1):289--300.

\bibitem[\protect\citeauthoryear{Burnetas and
  Katehakis}{1997}]{burnetas1997optimal}
Burnetas, A.~N., and Katehakis, M.~N.
\newblock 1997.
\newblock Optimal adaptive policies for markov decision processes.
\newblock {\em Mathematics of Operations Research} 22(1):222--255.

\bibitem[\protect\citeauthoryear{Caron \bgroup et al\mbox.\egroup
  }{2012}]{caron12}
Caron, S.; Kveton, B.; Lelarge, M.; and Bhagat, S.
\newblock 2012.
\newblock Leveraging side observations in stochastic bandits.
\newblock {\em UAI}.

\bibitem[\protect\citeauthoryear{Cesa-Bianchi and
  Lugosi}{2006}]{CesaBianchi-Lugosi:PLG}
Cesa-Bianchi, N., and Lugosi, G.
\newblock 2006.
\newblock {\em Prediction, Learning and Games}.
\newblock Cambridge Press.

\bibitem[\protect\citeauthoryear{Chapelle and Li}{2011}]{chapelle2011empirical}
Chapelle, O., and Li, L.
\newblock 2011.
\newblock An empirical evaluation of thompson sampling.
\newblock In {\em Advances in neural information processing systems},
  2249--2257.

\bibitem[\protect\citeauthoryear{Cohen, Hazan, and
  Koren}{2016}]{cohen2016online}
Cohen, A.; Hazan, T.; and Koren, T.
\newblock 2016.
\newblock Online learning with feedback graphs without the graphs.
\newblock In {\em {ICML} 2016}.

\bibitem[\protect\citeauthoryear{Condon and Karp}{2001}]{condon01}
Condon, A., and Karp, R.~M.
\newblock 2001.
\newblock Algorithms for graph partitioning on the planted partition model.
\newblock {\em Random Structures and Algorithms} 18(2):116--140.

\bibitem[\protect\citeauthoryear{David}{1968}]{david1968miscellanea}
David, H.
\newblock 1968.
\newblock Miscellanea: Gini's mean difference rediscovered.
\newblock {\em Biometrika} 55(3):573--575.

\bibitem[\protect\citeauthoryear{Garivier and Capp{\'e}}{2011}]{garivier2011kl}
Garivier, A., and Capp{\'e}, O.
\newblock 2011.
\newblock The kl-ucb algorithm for bounded stochastic bandits and beyond.
\newblock {\em arXiv preprint arXiv:1102.2490}.

\bibitem[\protect\citeauthoryear{Gini and Pearson}{1912}]{gini1912variabilita}
Gini, C., and Pearson, K.
\newblock 1912.
\newblock {\em Variabilit{\`a} e mutabilit{\`a}: contributo allo studio delle
  distribuzioni e delle relazioni statistiche. Fascicolo 1}.
\newblock tipografia di Paolo Cuppini.

\bibitem[\protect\citeauthoryear{Goh, Kahng, and Kim}{2001}]{goh2001universal}
Goh, K.-I.; Kahng, B.; and Kim, D.
\newblock 2001.
\newblock Universal behavior of load distribution in scale-free networks.
\newblock {\em Physical Review Letters} 87(27):278701.

\bibitem[\protect\citeauthoryear{Lai and Robbins}{1985}]{lai1985asymptotically}
Lai, T.~L., and Robbins, H.
\newblock 1985.
\newblock Asymptotically efficient adaptive allocation rules.
\newblock {\em Advances in applied mathematics} 6(1):4--22.

\bibitem[\protect\citeauthoryear{Mannor and Shamir}{2011}]{mannor2011bandits}
Mannor, S., and Shamir, O.
\newblock 2011.
\newblock From bandits to experts: On the value of side-observations.
\newblock In {\em Advances in Neural Information Processing Systems},
  684--692.

\bibitem[\protect\citeauthoryear{McSherry}{2001}]{mcsherry01}
McSherry, F.
\newblock 2001.
\newblock Spectral partitioning of random graphs.
\newblock In {\em Foundations of Computer Science, 2001. Proceedings. 42nd IEEE
  Symposium on},  529--537.
\newblock IEEE.

\bibitem[\protect\citeauthoryear{Ruan \bgroup et al\mbox.\egroup
  }{2004}]{ruan2004greedy}
Ruan, L.; Du, H.; Jia, X.; Wu, W.; Li, Y.; and Ko, K.-I.
\newblock 2004.
\newblock A greedy approximation for minimum connected dominating sets.
\newblock {\em Theoretical Computer Science} 329(1):325--330.

\bibitem[\protect\citeauthoryear{Russo and Roy}{2016}]{russo2014information}
Russo, D., and Roy, B.~V.
\newblock 2016.
\newblock An information-theoretic analysis of thompson sampling.
\newblock {\em Journal of Machine Learning Research} 17(68):1--30.

\bibitem[\protect\citeauthoryear{Scott}{2010}]{scott2010modern}
Scott, S.~L.
\newblock 2010.
\newblock A modern bayesian look at the multi-armed bandit.
\newblock {\em Applied Stochastic Models in Business and Industry}
  26(6):639--658.

\bibitem[\protect\citeauthoryear{Scott}{2015}]{scott2015multi}
Scott, S.~L.
\newblock 2015.
\newblock Multi-armed bandit experiments in the online service economy.
\newblock {\em Applied Stochastic Models in Business and Industry}
  31(1):37--45.

\bibitem[\protect\citeauthoryear{Thompson}{1933}]{thompson1933lou}
Thompson, W.
\newblock 1933.
\newblock {On the Likelihood that One Unknown Probability Exceeds Another in
  View of the Evidence of two Samples}.
\newblock {\em Biometrika} 25(3-4):285--294.

\bibitem[\protect\citeauthoryear{Yitzhaki and others}{2003}]{yitzhaki2003gini}
Yitzhaki, S., et~al.
\newblock 2003.
\newblock Gini’s mean difference: A superior measure of variability for
  non-normal distributions.
\newblock {\em Metron} 61(2):285--316.

\end{thebibliography}
\end{document}